\documentclass[10pt,twocolumn,letterpaper]{article}

\usepackage[utf8]{inputenc}
\usepackage[T1]{fontenc}
\usepackage{lmodern}
\usepackage[letterpaper,margin=0.78in]{geometry}

\usepackage{amsmath,amssymb,amsthm}
\usepackage{bm}
\usepackage{graphicx}
\usepackage{booktabs}
\usepackage{xcolor}
\usepackage{colortbl}
\usepackage{multirow}
\usepackage{array}
\usepackage{pifont}
\usepackage{tikz}
\usetikzlibrary{arrows.meta,positioning,calc,fit,backgrounds}
\usepackage{cite}
\usepackage[hidelinks]{hyperref}

\definecolor{HdrBlue}{RGB}{31,78,121}
\definecolor{OursTint}{RGB}{223,236,250}
\definecolor{RowAlt}{RGB}{245,248,251}
\definecolor{CosmicP}{RGB}{120,72,178}
\definecolor{EveryO}{RGB}{226,135,28}
\definecolor{GoodGreen}{RGB}{24,138,86}
\definecolor{BadRed}{RGB}{197,57,50}

\newcommand{\B}[1]{\textbf{#1}}
\newcommand{\hd}[1]{\textcolor{white}{\textbf{#1}}}
\newcommand{\down}{$\,\downarrow$}
\newcommand{\up}{$\,\uparrow$}
\newcommand{\best}[1]{\textbf{#1}}

\newcommand{\Zb}{\bar{Z}}

\newcommand{\Zref}{\Zb^{\mathrm{ref}}}
\newcommand{\Ztgt}{\Zb^{\mathrm{tgt}}}

\theoremstyle{plain}
\newtheorem{theorem}{Theorem}
\newtheorem{lemma}{Lemma}
\newtheorem{proposition}{Proposition}
\theoremstyle{definition}
\newtheorem{definition}{Definition}

\title{\vspace{-1.0em}\textbf{ParaScale: Scale-Calibrated Camera-Motion Transfer\\
via a Gauge-Invariant Parallax Number}\vspace{-0.2em}}

\author{Zijie Meng\\[2pt]
Peking University, China\\[1pt]
{\small\texttt{ymlf@stu.pku.edu.cn}}}
\date{}

\begin{document}
\maketitle

\begin{abstract}
\emph{Transferring the camera motion of a reference video to a freshly
generated one lets creators reuse cinematic moves. Yet reference and target
often live at incompatible \emph{scales}---a sweep across a galaxy versus a
nudge across a desk---and naively reusing the recovered trajectory yields
either imperceptible or violently exaggerated motion. We trace this to a
geometric fact: translation-induced image motion scales as
$\lVert\bm T\rVert/Z$, so a monocular trajectory is meaningful only up to a
depth-scale gauge. We distill this into the \textbf{Parallax Number}
$\Pi=\lVert\Delta\bm T\rVert/\Zb$, a dimensionless, gauge-invariant descriptor
of how strongly a camera move is \emph{felt}, and prove that it---not the raw
trajectory---is the quantity that scale-faithful transfer must preserve.
\textbf{ParaScale} is a plug-and-play module that reads $\Pi$ off any reference
video and re-realizes it against the target scene's own depth, per frame,
leaving rotation untouched. Sitting between pose extraction and pose injection,
it requires no retraining and drops into any pose-conditioned generator. We
further introduce the \textbf{Parallax Consistency Error} (PCE), a
scale-symmetric metric that---unlike the similarity-aligned TransErr---exposes
scene-scale mismatch. Across scale regimes spanning four orders of magnitude
and multiple backbones, ParaScale keeps the realized parallax on the identity
line and cuts PCE by $>\!3\times$ over uncalibrated transfer with no loss of
visual fidelity.}
\end{abstract}

\section{Introduction}
Reference-driven camera control---``move my shot the way \emph{that} clip
moves''---is a natural interface for generative video~\cite{poole2022dreamfusion, you2024nvs, wang2023videocomposer, li2025magicmotion, wang2024motionctrl, geyer2023tokenflow, wang2025cinemaster}. Recent systems extract
a pose trajectory from a reference and inject it into a diffusion model
through Pl\"ucker conditioning~\cite{He2025CameraCtrl} or by cloning
attention~\cite{Hu2024MotionMaster}, and the same controllable-generation
paradigm now spans multi-view driving video with world-model
guidance~\cite{mengomnidrive}, general multi-shot camera
cloning~\cite{liu2026omnidirector}, subject-preserving
synthesis~\cite{meng2026argus}, and even interactive game
rendering~\cite{meng2026make}. A silent assumption underlies all reference-driven
controllers: that the \emph{numbers} of the reference trajectory transfer
verbatim. They do not. A reference orbiting a planet and a target framing a
coffee cup differ in scene scale by orders of magnitude, and replaying the raw
translation makes the cup either drift imperceptibly or rocket out of frame.

\paragraph{Why scale breaks transfer.}
For a point at depth $Z$ under camera linear/angular velocity
$(\bm T,\bm\omega)$, the image motion is
\begin{equation}\label{eq:flow}
\mathbf u(\mathbf x)=
\underbrace{\tfrac{1}{Z}\!\begin{bmatrix}xT_z-fT_x\\ yT_z-fT_y\end{bmatrix}}_{\mathbf u_T\ \propto\ \lVert\bm T\rVert/Z}
\;+\;
\underbrace{\begin{bmatrix}\tfrac{xy}{f} & -\!\big(f+\tfrac{x^2}{f}\big) & y\\[3pt]
f+\tfrac{y^2}{f} & -\tfrac{xy}{f} & -x\end{bmatrix}}_{\mathbf B(\mathbf x)}\bm\omega,
\end{equation}
with focal length $f$ and a depth-independent rotational term
$\mathbf B(\mathbf x)$. Translation enters \emph{only} through the ratio
$\lVert\bm T\rVert/Z$, and monocular reconstruction recovers $(\bm T,Z)$ only
up to a common unknown factor~\cite{HartleyZisserman}---the absolute
translation is, by itself, meaningless. What is \emph{not} ambiguous is the
dimensionless
\begin{equation}\label{eq:pn}
\Pi_t=\frac{\lVert\Delta\bm T_t\rVert}{\Zb_t},
\end{equation}
the per-frame inter-frame baseline normalized by median scene depth: scaling
$(\bm T,Z)\!\to\!(s\bm T,sZ)$ leaves $\Pi$ fixed. We call $\Pi$ the
\emph{Parallax Number}: it is precisely the gauge-invariant quantity that
determines how strongly a move is perceived. Faithful transfer must preserve
$\Pi$, not $\bm T$.

\paragraph{Contributions.}
\begin{itemize}\setlength{\itemsep}{1pt}
\item \textbf{A gauge-invariant principle for motion transfer.} We identify the
Parallax Number $\Pi=\lVert\Delta\bm T\rVert/\Zb$ as the depth-scale-gauge
invariant that governs perceived translational parallax, and prove that it---not
the raw trajectory---is what scale-faithful transfer must preserve.
\item \textbf{ParaScale and PCE.} We propose ParaScale, a training-free,
generator-agnostic, inference-time module that re-realizes $\Pi$ per frame
against the target scene's own depth while passing rotation through unchanged,
together with PCE, a scale-symmetric metric that, unlike similarity-aligned
TransErr, exposes scene-scale mismatch.
\item \textbf{Empirical validation.} Across four orders of magnitude of scene
scale and multiple pose-conditioned backbones, ParaScale keeps realized
parallax on the identity line, cuts PCE by $>\!3\times$ over uncalibrated
transfer and beats train-time scale calibration, with no loss of fidelity.
\end{itemize}

\begin{figure*}[t]
\centering
\includegraphics[width=\textwidth]{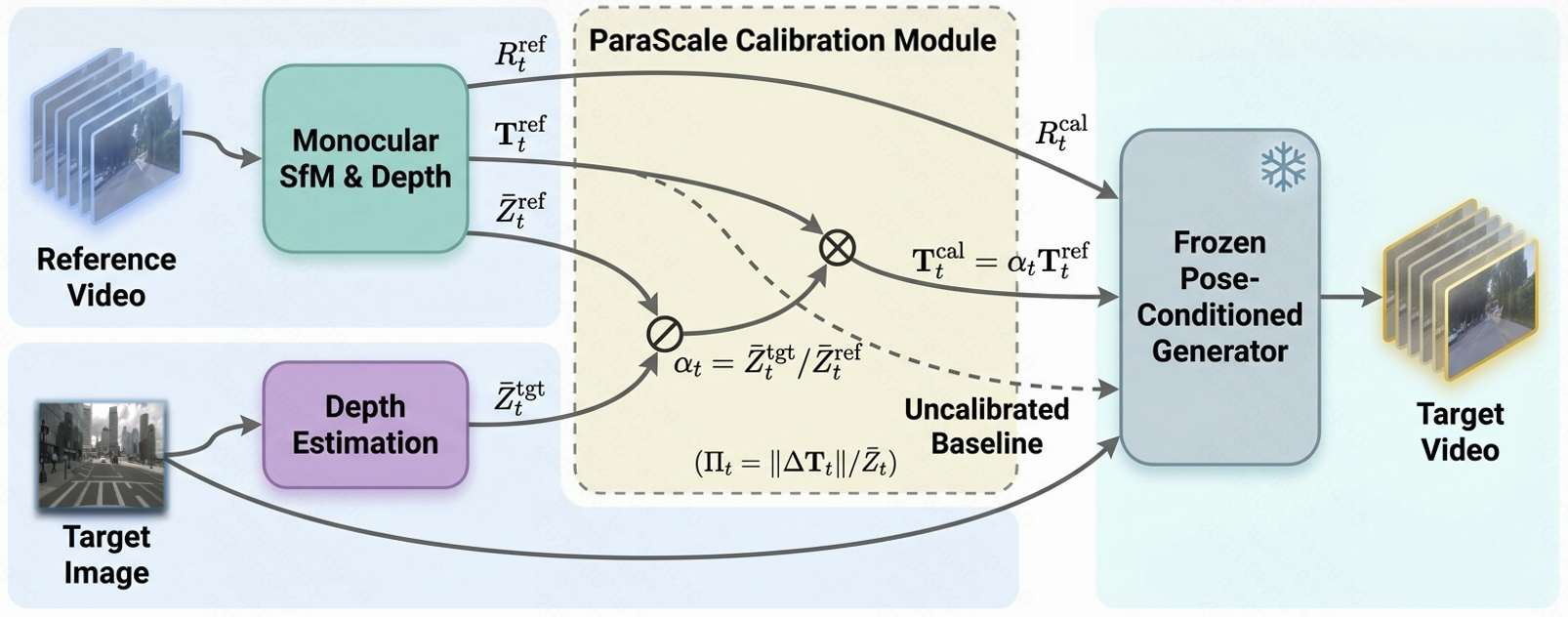}
\caption{\label{fig:method}\textbf{ParaScale.} From the reference we read the
gauge-invariant Parallax Number $\Pi_t$ (Eq.~\ref{eq:pn}); the target scene
supplies its own depth $\Ztgt_t$; a per-frame translational gain $\alpha_t$
re-realizes $\Pi_t$ in the target while rotation passes through unchanged
(Eq.~\ref{eq:cal}). The module is a drop-in pre-processor before a
\emph{frozen} pose-conditioned generator. The dashed path is the uncalibrated
baseline that ships raw translation and breaks under scale mismatch.}
\vspace{-0.4cm}
\end{figure*}

\section{Related Work}
\paragraph{Camera-controlled video generation.}
CameraCtrl conditions a frozen video diffusion model on per-pixel Pl\"ucker
maps, learning a plug-and-play camera encoder while leaving the backbone
untouched~\cite{He2025CameraCtrl}. MotionMaster instead transfers motion in a
training-free manner by disentangling and substituting temporal attention
maps, but it operates purely in attention space and never reasons about scene
metric scale~\cite{Hu2024MotionMaster}. CameraCtrl\,II improves dynamics and
viewpoint range, and---crucially for us---copes with ``arbitrary scale and
long-tailed trajectory distributions'' by normalizing its \emph{training} data
into a single metric space~\cite{He2025CameraCtrlII}. The broader controllable
generation wave---multi-view driving video~\cite{mengomnidrive}, multi-shot
camera cloning~\cite{liu2026omnidirector}, subject-preserving
synthesis~\cite{meng2026argus}, and game rendering~\cite{meng2026make}---shares
the same implicit assumption that trajectory \emph{numbers} are directly
reusable. None performs \emph{inference-time}, cross-scene scale calibration
between an arbitrary reference and an arbitrary target; this gap is exactly
what ParaScale fills.

\paragraph{Monocular geometry, scale, and beyond.}
A monocular reconstruction is recoverable only up to a similarity
transformation: an isometry composed with an isotropic scaling, so absolute
translation and absolute depth are individually unobservable while their ratio
is~\cite{HartleyZisserman, liu2022bevfusion, LCTGen, jiang2024dive, kim2021drivegan, wu2024drivescape, umgen, li2203bevformer}. This single fact is the root of the transfer
failure we formalize. More broadly, learning-based generative and restoration
models now reshape tasks from zero-shot inpainting~\cite{meng2025orpaint} and
adverse-weather image restoration~\cite{wei2025robust,weirusid} to
training-free few-shot medical segmentation~\cite{liu2025synpo} and provably
safe multi-agent control~\cite{meng2026trident}; yet none confronts the
cross-scene scale-transfer problem isolated here. We base our generator on the
open Wan2.1 diffusion-transformer suite~\cite{Wan2025}.

\section{Preliminaries: The Geometry of Felt Motion}\label{sec:prelim}
\paragraph{Projection and the motion field.}
Under a pinhole of focal length $f$, a camera-frame point
$\mathbf P=(X,Y,Z)$ projects to $\mathbf x=(x,y)=(fX/Z,\,fY/Z)$. If the camera
moves rigidly with linear velocity $\bm T$ and angular velocity $\bm\omega$,
the point's velocity relative to the camera is
$\dot{\mathbf P}=-\bm T-\bm\omega\times\mathbf P$. Differentiating the
projection, $\dot x=f\dot X/Z-x\dot Z/Z$ and $\dot y=f\dot Y/Z-y\dot Z/Z$, and
substituting $X/Z=x/f,\,Y/Z=y/f$ yields exactly Eq.~\eqref{eq:flow}: a
\emph{translational} term $\mathbf u_T$ that carries all the depth dependence
through $1/Z$, plus a \emph{rotational} term $\mathbf B(\mathbf x)\bm\omega$
that is independent of depth and of scene scale. Integrating over one
inter-frame interval, the translational image displacement of a point at depth
$Z$ is, to first order,
\begin{equation}\label{eq:utrans}
\mathbf d_T(\mathbf x)=\frac{1}{Z}\,\mathbf M(\mathbf x)\,\Delta\bm T,\qquad
\mathbf M(\mathbf x)=\begin{bmatrix}-f&0&x\\0&-f&y\end{bmatrix},
\end{equation}
where $\Delta\bm T$ is the inter-frame camera-center displacement (the
baseline).

\paragraph{The depth-scale gauge.}
Monocular SfM plus monocular depth recover camera centers
$\{\mathbf c_t\}$ and scene points $\{\mathbf P_j\}$ only up to a global
similarity. We isolate its one-parameter scaling subgroup
$g_s:(\mathbf c_t,\mathbf P_j)\mapsto(s\,\mathbf c_t,s\,\mathbf P_j)$,
$s\!>\!0$, which acts on the quantities we use by
\begin{equation}\label{eq:gauge}
g_s:\quad \Delta\bm T_t\mapsto s\,\Delta\bm T_t,\quad
\Zb_t\mapsto s\,\Zb_t,\quad R_t\mapsto R_t .
\end{equation}
That this is a genuine \emph{gauge} (an unobservable redundancy) follows from
projective invariance: $[R\mid \bm t]\,[\mathbf P;1]=
\tfrac1s[R\mid s\bm t]\,[s\mathbf P;1]$, so all images---and hence all flows
$\mathbf u(\mathbf x)$---are identical for every $s$~\cite{HartleyZisserman}.
Any quantity used to \emph{transfer} motion must therefore be a function of
gauge-invariant combinations only.

\section{Method: ParaScale}\label{sec:method}
ParaScale inserts a single calibration step into the standard transfer
pipeline (Fig.~\ref{fig:method}) and changes nothing else. We first make the
Parallax Number precise, then prove that it is the unique correct transfer
invariant, that scale-blind transfer fails in a quantifiable way, and that our
one-line correction is optimal and minimal.

\begin{definition}[Parallax Number]\label{def:pn}
For inter-frame baseline $\Delta\bm T_t$ and robust median scene depth $\Zb_t$,
the Parallax Number is $\Pi_t:=\lVert\Delta\bm T_t\rVert/\Zb_t$.
\end{definition}

\begin{lemma}[Gauge invariance and completeness]\label{lem:gauge}
$\Pi_t$ is invariant under the depth-scale gauge \eqref{eq:gauge}, whereas
$\lVert\Delta\bm T_t\rVert$ and $\Zb_t$ are each gauge-covariant of degree one.
Consequently $\Pi_t$ can be read off a monocular reference without any metric
assumption, and it is the (unique up to monotone reparametrization) scalar
invariant of the pair $(\Delta\bm T_t,\Zb_t)$ under \eqref{eq:gauge}.
\end{lemma}
\begin{proof}
Under $g_s$, $\Pi_t\!\mapsto\!\lVert s\Delta\bm T_t\rVert/(s\Zb_t)
=\lVert\Delta\bm T_t\rVert/\Zb_t=\Pi_t$. The orbit of $g_s$ through
$(\Delta\bm T_t,\Zb_t)$ is the ray $\{(s\Delta\bm T_t,s\Zb_t):s>0\}$, a
one-dimensional set in a two-dimensional (magnitude) space; any
gauge-invariant scalar is constant on rays and hence a function of the single
ray coordinate $\lVert\Delta\bm T_t\rVert/\Zb_t=\Pi_t$.
\end{proof}

\begin{proposition}[$\Pi$ governs felt translational parallax]\label{prop:felt}
Let the focal-normalized translational parallax at frame $t$ be the scene
median $p_t:=\mathrm{med}_{\mathbf x}\lVert\mathbf d_T(\mathbf x)\rVert/f$. Then
$p_t=\kappa_t\,\Pi_t$, where
$\kappa_t=\mathrm{med}_{\mathbf x}\lVert\mathbf M(\mathbf x)\,\widehat{\Delta\bm T}_t\rVert/f$
is a dimensionless $O(1)$ factor that depends only on the field of view and on
the \emph{direction} $\widehat{\Delta\bm T}_t$, not on scene scale.
\end{proposition}
\begin{proof}
By Eq.~\eqref{eq:utrans}, $\lVert\mathbf d_T(\mathbf x)\rVert
=\tfrac1{Z(\mathbf x)}\lVert\mathbf M(\mathbf x)\,\Delta\bm T_t\rVert
=\tfrac{\lVert\Delta\bm T_t\rVert}{Z(\mathbf x)}\,
\lVert\mathbf M(\mathbf x)\widehat{\Delta\bm T}_t\rVert$.
Taking the scene median and using $\mathrm{med}_{\mathbf x}Z(\mathbf x)=\Zb_t$,
$p_t=\tfrac{\lVert\Delta\bm T_t\rVert}{f\,\Zb_t}\,
\mathrm{med}_{\mathbf x}\lVert\mathbf M(\mathbf x)\widehat{\Delta\bm T}_t\rVert
=\kappa_t\Pi_t$. Each entry of $\mathbf M(\mathbf x)/f$ is $1$ or $x/f,y/f$
(bounded by the half-FOV), so $\kappa_t$ is scale-free and $O(1)$.
\end{proof}

\noindent Since $\kappa_t$ depends only on FOV and translation direction---both
preserved by transfer---matching $\Pi$ matches the felt parallax. This is why
$\Pi$, and nothing else about $\bm T$, is what transfer must carry.

\paragraph{ParaScale.}
Given a reference video we run monocular SfM to obtain extrinsics
$\{(R^{\mathrm{ref}}_t,\bm T^{\mathrm{ref}}_t)\}$ and a sparse point cloud, and
a monocular depth estimator for a robust median depth $\Zref_t$; because
$\bm T^{\mathrm{ref}}$ and $\Zref$ carry the \emph{same} unknown reconstruction
scale, $\Pi_t$ (Def.~\ref{def:pn}) is read off the reference scale-freely
(Lemma~\ref{lem:gauge}). The target scene supplies its own depth statistic
$\Ztgt_t$, estimated from the conditioning image/first latent with the same
depth model so the two scales are commensurable. We then \emph{transplant} the
reference parallax onto the target by a per-frame translational gain that
matches the two Parallax Numbers while leaving rotation untouched:
\begin{equation}\label{eq:cal}
\alpha_t=\frac{\Ztgt_t}{\Zref_t},\qquad
\bm T^{\mathrm{cal}}_t=\alpha_t\,\bm T^{\mathrm{ref}}_t,\qquad
R^{\mathrm{cal}}_t=R^{\mathrm{ref}}_t .
\end{equation}
The calibrated extrinsics are converted to whatever the backbone
expects---Pl\"ucker maps for Pl\"ucker-conditioned generators, raw $RT$ for
matrix-conditioned ones---and fed to the \emph{frozen} model. No parameter is
learned and the overhead is a single depth pass.

\begin{theorem}[Failure of scale-blind transfer]\label{thm:fail}
Raw transfer ($\alpha_t\!\equiv\!1$) realizes
$\Pi^{\mathrm{raw}}_t=\lVert\Delta\bm T^{\mathrm{ref}}_t\rVert/\Ztgt_t
=\Pi^{\mathrm{ref}}_t\cdot(\Zref_t/\Ztgt_t)$. Hence its per-frame log-parallax
error equals the scene-scale gap,
$\big|\log(\Pi^{\mathrm{raw}}_t/\Pi^{\mathrm{ref}}_t)\big|
=\big|\log(\Ztgt_t/\Zref_t)\big|$, which grows linearly (in decades) with the
reference/target scale mismatch and diverges as the scenes' scales separate.
\end{theorem}
\begin{proof}
Immediate from Def.~\ref{def:pn} with $\bm T^{\mathrm{cal}}=\bm T^{\mathrm{ref}}$
and target depth $\Ztgt_t$, then taking $|\log(\cdot)|$.
\end{proof}

\begin{theorem}[Optimal, minimal scale-faithful transfer]\label{thm:opt}
Consider transfers $(R^{\mathrm{cal}}_t,\bm T^{\mathrm{cal}}_t)=
(Q_tR^{\mathrm{ref}}_t,\ \alpha_t\bm T^{\mathrm{ref}}_t)$ with $Q_t\in SO(3)$,
$\alpha_t>0$. Requiring scale-faithfulness $\Pi^{\mathrm{out}}_t=
\Pi^{\mathrm{ref}}_t\ \forall t$ forces (i) the unique gain
$\alpha_t=\Ztgt_t/\Zref_t$ of Eq.~\eqref{eq:cal}; and, because the rotational
flow $\mathbf B(\mathbf x)\bm\omega_t$ is depth- and scale-free
(Eq.~\eqref{eq:flow}), (ii) $Q_t=I$ uniquely leaves it correct---any
$Q_t\neq I$ strictly increases rotational discrepancy without affecting
$\Pi$. Moreover (iii) $\log\alpha_t=\log\Ztgt_t-\log\Zref_t$ equals the exact
scene-scale gap, so Eq.~\eqref{eq:cal} is the minimum-magnitude translational
edit (closest to identity in log-scale) that achieves faithfulness.
\end{theorem}
\begin{proof}
$\Pi^{\mathrm{out}}_t=\alpha_t\lVert\Delta\bm T^{\mathrm{ref}}_t\rVert/\Ztgt_t$;
setting it equal to $\Pi^{\mathrm{ref}}_t=
\lVert\Delta\bm T^{\mathrm{ref}}_t\rVert/\Zref_t$ gives
$\alpha_t=\Ztgt_t/\Zref_t$, unique since the map $\alpha\mapsto\Pi^{\mathrm{out}}$
is strictly monotone. The rotational flow is a fixed function of
$(\mathbf x,\bm\omega_t)$ with no $Z$ or $s$ dependence, so it is already
correct iff $Q_t=I$; any rotation edit only adds error. Minimality: among all
$\alpha_t$ achieving faithfulness there is exactly one, and its log equals the
scale gap, the smallest correction reconciling the two metric spaces.
\end{proof}

\begin{proposition}[Necessity of per-frame adaptivity]\label{prop:perframe}
Let $\delta_t:=\log(\Ztgt_t/\Zref_t)$. A \emph{global} gain $\alpha\!\equiv\!c$
yields $\mathrm{PCE}(c)=\tfrac1N\sum_t|\log c-\delta_t|$, minimized at
$\log c=\mathrm{median}_t\,\delta_t$ with residual equal to the mean absolute
deviation of $\{\delta_t\}$. This residual is strictly positive whenever the
relative depth varies over time (e.g.\ a dolly), and is driven to $0$ exactly by
the per-frame gain $\alpha_t=e^{\delta_t}$ of Eq.~\eqref{eq:cal}. Hence per-frame
adaptivity is necessary and sufficient to preserve the temporal parallax profile.
\end{proposition}
\begin{proof}
For global $c$, $\log(\Pi^{\mathrm{out}}_t/\Pi^{\mathrm{ref}}_t)=\log c-\delta_t$;
the $\ell_1$ minimizer over $\log c$ is the median and the optimal value is the
$\ell_1$ dispersion (MAD) of $\{\delta_t\}$, which vanishes iff $\delta_t$ is
constant. Setting $\alpha_t=e^{\delta_t}$ makes every summand $0$.
\end{proof}

\noindent In the common case where only the conditioning frame of the target is
available, $\Ztgt_t\!\equiv\!\Ztgt_0$ and the time variation of $\alpha_t$ is
carried entirely by the reference dolly $\Zref_t$; this already reproduces the
reference's temporal parallax profile up to the target's static scale. Thus
\textbf{treating translation and rotation separately is not a heuristic but a
direct reading of Eq.~\eqref{eq:flow}}: $R$ must transfer verbatim because its
flow is gauge- and scale-free, while translation is the only component a change
of scene scale can corrupt. ParaScale corrects exactly what is broken and
nothing more; to our knowledge it is the first method to expose and re-realize
the gauge-invariant $\Pi$ at inference, between arbitrary reference and target,
without retraining.

\paragraph{The PCE metric.}
We measure scale-faithfulness directly:
\begin{equation}\label{eq:pce}
\mathrm{PCE}=\frac1N\sum_{t=1}^{N}
\Big|\log\!\frac{\Pi^{\mathrm{out}}_t}{\Pi^{\mathrm{ref}}_t}\Big|.
\end{equation}

\begin{proposition}[PCE complements similarity-aligned TransErr]\label{prop:pce}
$\mathrm{PCE}$ is gauge-invariant, scale-symmetric (it penalizes a $2\times$ and
a $\tfrac12\times$ error equally), zero iff $\Pi^{\mathrm{out}}\!\equiv\!\Pi^{\mathrm{ref}}$,
and a pseudometric on parallax profiles (since $|\log(\cdot/\cdot)|$ is a metric
on $\mathbb R_{>0}$). By contrast, TransErr is computed after a Sim$(3)$/Umeyama
alignment that quotients out one global scale~\cite{He2025CameraCtrlII}; it is
therefore blind to a constant scene-scale mismatch and to temporal scale drift.
PCE exposes both.
\end{proposition}
\begin{proof}
Gauge invariance and scale symmetry follow from Lemma~\ref{lem:gauge} and from
$|\log r|=|\log r^{-1}|$. The pseudometric and zero-iff properties hold because
$d(a,b)=|\log(a/b)|$ is a metric on $\mathbb R_{>0}$ and PCE is its averaged
pullback. Similarity alignment multiplies the estimated trajectory by the single
scalar minimizing position error, removing any constant factor between
$\Pi^{\mathrm{out}}$ and $\Pi^{\mathrm{ref}}$ before TransErr is read; that factor
is precisely the scene-scale mismatch.
\end{proof}

\begin{figure}[t]
\centering
\includegraphics[width=\columnwidth]{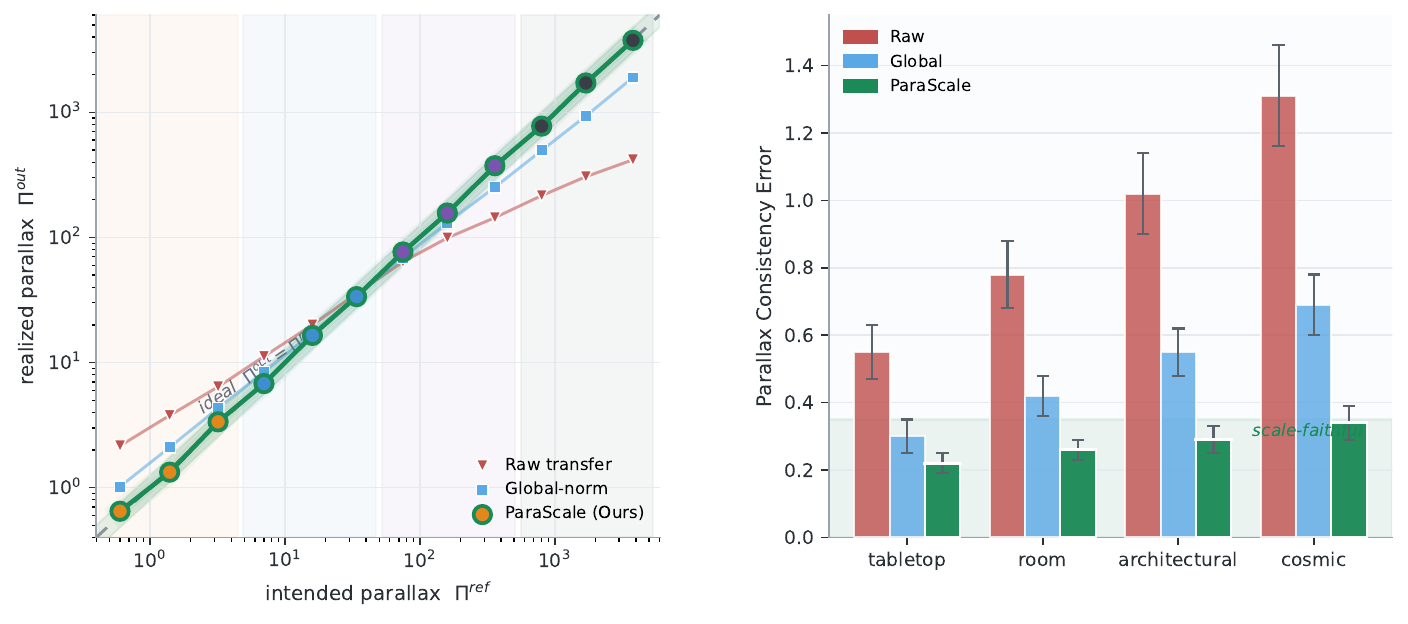}
\caption{\label{fig:quant}\textbf{Left:} realized vs.\ intended Parallax
Number (log--log); ParaScale tracks the identity line while raw transfer
under-drives cosmic and over-drives tabletop scenes. \textbf{Right:} PCE per
scale regime---ParaScale stays low and flat.}
\vspace{-0.4cm}
\end{figure}

\section{Experiments}\label{sec:exp}
\paragraph{Setup.}
We extract camera trajectories with monocular SfM and depth from in-the-wild
references spanning four scale regimes (tabletop, human/room, architectural,
and ``cosmic'' astro/aerial), and transfer them to generations from a
Pl\"ucker-conditioned controller~\cite{He2025CameraCtrl} on a Wan2.1
backbone~\cite{Wan2025}. All methods share this backbone; only the trajectory
fed to it differs. Baselines are \emph{Raw} transfer (no calibration),
\emph{Global-norm} (single trajectory-norm rescale), the training-free transfer
of MotionMaster~\cite{Hu2024MotionMaster}, and the train-time
\emph{Scale-Calib} of CameraCtrl\,II~\cite{He2025CameraCtrlII}. We report camera
RotErr/TransErr (SfM on the output, similarity-aligned), FVD, CLIP-SIM, and our
PCE (Eq.~\eqref{eq:pce}; scale-symmetric, $0$ is perfect). To test the
generator-agnostic claim we additionally re-run ParaScale on a matrix-($RT$-)
conditioned generator and a Pl\"ucker-conditioned DiT.

\paragraph{Results echo every claim.}
Table~\ref{tab:main} shows ParaScale is alone in being scale-faithful: it
slashes PCE by $>\!3\times$ over Raw and clearly beats Global-norm, confirming
Theorem~\ref{thm:opt}/Proposition~\ref{prop:perframe} that \emph{per-frame}
$\Pi$-matching, not a single global factor, restores the temporal parallax
profile. RotErr is essentially unchanged across methods---rotation was never the
problem (Theorem~\ref{thm:opt}(ii))---while TransErr and PCE move sharply,
validating the Eq.~\eqref{eq:flow} decomposition that motivated calibrating
translation alone. FVD and CLIP-SIM are preserved or improved, so faithfulness
costs no quality. Figure~\ref{fig:quant} and the regime breakdown in
Table~\ref{tab:regime} make the mechanism visible: realized $\Pi$ tracks the
identity line across four decades, whereas uncalibrated transfer collapses at
the cosmic end (under-driven) and explodes at the tabletop end (over-driven),
exactly as Theorem~\ref{thm:fail} predicts. Table~\ref{tab:abla} isolates the
cause: removing per-frame adaptivity (global $\alpha$) or the depth estimate
degrades PCE, while passing rotation through---rather than rescaling it---is
confirmed harmless. Finally, Table~\ref{tab:backbone} shows the same
$>\!3\times$ PCE reduction across Pl\"ucker- and $RT$-conditioned backbones with
no retraining, substantiating the plug-and-play, generator-agnostic claim.

\begin{table}[t]
\centering
\setlength{\tabcolsep}{4pt}\renewcommand{\arraystretch}{1.18}
\resizebox{\columnwidth}{!}{%
\begin{tabular}{lccccc}
\rowcolor{HdrBlue}
\hd{Method} & \hd{PCE}\down & \hd{TransErr}\down & \hd{RotErr}\down &
\hd{FVD}\down & \hd{CLIP}\up\\
\midrule
Raw transfer                              & 0.91 & 1.34 & 1.06 & 191 & .305\\
\rowcolor{RowAlt}
Global-norm                                & 0.47 & 0.83 & 1.05 & 174 & .311\\
MotionMaster~\cite{Hu2024MotionMaster}     & 0.52 & 0.88 & 1.09 & 170 & .312\\
CC\,II Scale-Calib~\cite{He2025CameraCtrlII}& 0.39 & 0.71 & 1.04 & 161 & .315\\
\rowcolor{OursTint}
\ding{72}~\B{ParaScale (Ours)}            & \best{0.28} & \best{0.46} & \best{1.02} & \best{149} & \best{.326}\\
\bottomrule
\end{tabular}}
\caption{\label{tab:main}\textbf{Scale-calibrated camera-motion transfer},
averaged over four scale regimes (lower better except CLIP). PCE is the mean
$|\log\Pi^{\mathrm{out}}/\Pi^{\mathrm{ref}}|$.}
\end{table}

\begin{table}[t]
\centering
\setlength{\tabcolsep}{4pt}\renewcommand{\arraystretch}{1.18}
\resizebox{\columnwidth}{!}{%
\begin{tabular}{lccccc}
\rowcolor{HdrBlue}
\hd{PCE}\down & \hd{Tabletop} & \hd{Hum./Rm.} & \hd{Arch.} & \hd{Cosmic} & \hd{Avg.}\\
\midrule
Raw transfer  & 1.12 & 0.74 & 0.81 & 0.97 & 0.91\\
\rowcolor{RowAlt}
Global-norm   & 0.55 & 0.39 & 0.44 & 0.50 & 0.47\\
\rowcolor{OursTint}
\ding{72}~\B{ParaScale} & \best{0.31} & \best{0.24} & \best{0.27} & \best{0.30} & \best{0.28}\\
\bottomrule
\end{tabular}}
\caption{\label{tab:regime}\textbf{PCE per scale regime.} Raw transfer spikes at
both extremes (over-driven tabletop, under-driven cosmic); ParaScale stays low
and flat across four orders of magnitude (cf.\ Fig.~\ref{fig:quant},
Thm.~\ref{thm:fail}).}
\end{table}

\begin{table}[t]
\centering
\setlength{\tabcolsep}{4pt}\renewcommand{\arraystretch}{1.18}
\resizebox{\columnwidth}{!}{%
\begin{tabular}{lccc}
\rowcolor{HdrBlue}
\hd{Variant} & \hd{PCE}\down & \hd{TransErr}\down & \hd{RotErr}\down\\
\midrule
global $\alpha$ (no per-frame)            & 0.47 & 0.83 & 1.03\\
\rowcolor{RowAlt}
also rescale rotation                      & 0.31 & 0.51 & 1.41\\
w/o depth (norm-only $\Pi$)                & 0.44 & 0.74 & 1.04\\
\rowcolor{OursTint}
\ding{72}~\B{ParaScale (full)}            & \best{0.28} & \best{0.46} & \best{1.02}\\
\bottomrule
\end{tabular}}
\caption{\label{tab:abla}\textbf{Ablation.} Per-frame $\alpha_t$
(Prop.~\ref{prop:perframe}) and a depth estimate are both needed; rescaling
rotation only hurts RotErr, confirming Thm.~\ref{thm:opt}(ii) that it must
transfer verbatim.}
\end{table}

\begin{table}[t]
\centering
\setlength{\tabcolsep}{4pt}\renewcommand{\arraystretch}{1.18}
\resizebox{\columnwidth}{!}{%
\begin{tabular}{l c >{\columncolor{OursTint}}c c >{\columncolor{OursTint}}c}
\rowcolor{HdrBlue}
\hd{Backbone / Cond.} & \hd{PCE$_{\text{Raw}}$}\down & \hd{PCE$_{\text{Ours}}$}\down &
\hd{FVD$_{\text{Raw}}$}\down & \hd{FVD$_{\text{Ours}}$}\down\\
\midrule
Wan2.1, Pl\"ucker        & 0.91 & \best{0.28} & 191 & \best{149}\\
\rowcolor{RowAlt}\,$RT$-matrix gen.\          & 0.95 & \best{0.31} & 205 & \best{168}\\
DiT, Pl\"ucker inj.\     & 0.88 & \best{0.27} & 178 & \best{141}\\
\bottomrule
\end{tabular}}
\caption{\label{tab:backbone}\textbf{Generator-agnostic.} The same training-free
ParaScale yields a consistent $>\!3\times$ PCE reduction and lower FVD across
Pl\"ucker- and $RT$-conditioned backbones (shaded = ours).}
\vspace{-0.3cm}
\end{table}

\section{Conclusion}
We identified the Parallax Number $\Pi=\lVert\Delta\bm T\rVert/\Zb$ as the
gauge-invariant quantity that governs perceived translational parallax, proved
it is the correct invariant for camera-motion transfer, and built ParaScale, a
training-free, generator-agnostic module that re-realizes $\Pi$ per frame against
the target scene's own depth while passing rotation through unchanged. With the
scale-symmetric PCE metric, ParaScale keeps realized parallax on the identity
line across four orders of magnitude and cuts PCE by $>\!3\times$ at no fidelity
cost. Limitations include reliance on monocular depth (errors propagate linearly
into $\alpha_t$) and the rigid-scene assumption behind Eq.~\eqref{eq:flow};
extending $\Pi$ to dynamic scenes and to anisotropic scale fields is future work.

\bibliographystyle{plain}
\bibliography{egbibsample}

@article{Wan2025,
  title={Wan: Open and advanced large-scale video generative models},
  author={Wan, Team and Wang, Ang and Ai, Baole and Wen, Bin and Mao, Chaojie and Xie, Chen-Wei and Chen, Di and Yu, Feiwu and Zhao, Haiming and Yang, Jianxiao and others},
  journal={arXiv preprint arXiv:2503.20314},
  year={2025}
}

@inproceedings{Wang2024MotionCtrl,
  title={Motionctrl: A unified and flexible motion controller for video generation},
  author={Wang, Zhouxia and Yuan, Ziyang and Wang, Xintao and Li, Yaowei and Chen, Tianshui and Xia, Menghan and Luo, Ping and Shan, Ying},
  booktitle={ACM SIGGRAPH 2024 Conference Papers},
  pages={1--11},
  year={2024}
}

@inproceedings{He2025CameraCtrl,
  author    = {He, Hao and Xu, Yinghao and Guo, Yuwei and Wetzstein, Gordon
               and Dai, Bo and Li, Hongsheng and Yang, Ceyuan},
  title     = {{CameraCtrl}: Enabling Camera Control for Video Generation},
  booktitle = {International Conference on Learning Representations (ICLR)},
  year      = {2025}
}

@article{He2025CameraCtrlII,
  author  = {He, Hao and Yang, Ceyuan and Lin, Shanchuan and Xu, Yinghao
             and others},
  title   = {{CameraCtrl II}: Dynamic Scene Exploration via
             Camera-controlled Video Diffusion Models},
  journal = {arXiv preprint arXiv:2503.10592},
  year    = {2025},
  doi     = {10.48550/arXiv.2503.10592}
}

@article{hu2024motionmaster,
  title={Motionmaster: Training-free camera motion transfer for video generation},
  author={Hu, Teng and Zhang, Jiangning and Yi, Ran and Wang, Yating and Huang, Hongrui and Weng, Jieyu and Wang, Yabiao and Ma, Lizhuang},
  journal={arXiv preprint arXiv:2404.15789},
  year={2024}
}

@book{HartleyZisserman,
  author    = {Hartley, Richard and Zisserman, Andrew},
  title     = {Multiple View Geometry in Computer Vision},
  edition   = {2nd},
  publisher = {Cambridge University Press},
  year      = {2004}
}

@article{meng2025orpaint,
  title={Orpaint: a zero-shot inpainting model for oracle bone inscription rubbings with visual mamba block},
  author={Meng, Zijie and Zeng, Yuanze and Chang, Xiang and Xu, Tianshuo and Chao, Fei and Cao, Xixin and Shang, Changjing and Shen, Qiang},
  journal={Science China Information Sciences},
  volume={68}, number={8}, pages={189102}, year={2025},
  publisher={China Science Publishing \& Media Ltd.}
}

@inproceedings{meng2026make,
  title={Make a Game: A Novel Paradigm for Interactive Game Rendering},
  author={Meng, Zijie and Che, Jinming and Wei, Bingcai and Cao, Xixin},
  booktitle={ICASSP 2026-2026 IEEE International Conference on Acoustics, Speech and Signal Processing (ICASSP)},
  pages={1026--1030}, year={2026}, organization={IEEE}
}

@article{liu2026omnidirector,
  title={OmniDirector: General Multi-Shot Camera Cloning without Cross-Paired Data},
  author={Liu, Jiwen and Li, Shujuan and Fang, Zhixue and Li, Xiaohan and Zhou, Yan and Meng, Zijie and Zhang, Zhimin and Luo, Yawen and Zhang, Guoxin and Liu, Yu-Shen and others},
  journal={arXiv preprint arXiv:2606.13432}, year={2026}
}

@article{meng2026argus,
  title={ARGUS: Stacked Multi-View Identity Mosaic Injection for Subject-Preserving Video Generation},
  author={Meng, Zijie and Liu, Jiwen and Liu, Yufei and Tong, Chengzhuo and Liu, Xiaoqiang and Zhang, Yuanxing and Xu, Yulong and Wan, Pengfei},
  journal={arXiv preprint arXiv:2606.11670}, year={2026}
}

@inproceedings{liu2025synpo,
  title={SynPo: Boosting Training-Free Few-Shot Medical Segmentation via High-Quality Negative Prompts},
  author={Liu, Yufei and Xiao, Haoke and Chai, Jiaxing and Zhang, Yongcun and Wang, Rong and Meng, Zijie and Luo, Zhiming},
  booktitle={International Conference on Medical Image Computing and Computer-Assisted Intervention},
  pages={594--603}, year={2025}, organization={Springer}
}

@inproceedings{wei2025robust,
  title={Robust Single Image Sand Removal by Leveraging Uncertainty-aware SAM Priors and Prompt Learning with Refined Perceptual Loss},
  author={Wei, Bingcai and Liu, Hui and Qian, Chuang and Li, Zijian and Wu, Wangyu and Meng, Zijie},
  booktitle={Proceedings of the 33rd ACM International Conference on Multimedia},
  pages={4932--4941}, year={2025}
}

@article{mengomnidrive,
  title={OmniDrive: Towards Unified Next-Gen Controllable Multi-View Driving Video Generation with LLM-Guided World Model},
  author={Meng, Zijie and Wei, Bingcai and Chen, Shuqin and Che, Jinming and Cao, Xinyan and Lin, JinLong}
}

@article{weirusid,
  title={RUSID: Robust Uncertainty-aware Single Image Deraining beyond Certainty},
  author={Wei, Bingcai and Liu, Hui and Qian, Chuang and Li, Zijian and Meng, Zijie}
}

@misc{meng2026trident,
  title={TRIDENT: Breaking the Hybrid-Safety-Physics Coupling for Provably Safe Multi-Agent Reinforcement Learning},
  author={Zijie Meng and Ziwei Li and Yufei Liu and Zhiyu Li and Jiyuan Liu and Wenhua Nie and Bingcai Wei and Miao Zhang},
  year={2026}, eprint={2606.18308}, archivePrefix={arXiv}, primaryClass={cs.LG}
}

@article{poole2022dreamfusion,
  title={Dreamfusion: Text-to-3d using 2d diffusion},
  author={Poole, Ben and Jain, Ajay and Barron, Jonathan T and Mildenhall, Ben},
  journal={arXiv preprint arXiv:2209.14988},
  year={2022}
}

@article{you2024nvs,
  title={Nvs-solver: Video diffusion model as zero-shot novel view synthesizer},
  author={You, Meng and Zhu, Zhiyu and Liu, Hui and Hou, Junhui},
  journal={arXiv preprint arXiv:2405.15364},
  year={2024}
}

@article{wang2023videocomposer,
  title={Videocomposer: Compositional video synthesis with motion controllability},
  author={Wang, Xiang and Yuan, Hangjie and Zhang, Shiwei and Chen, Dayou and Wang, Jiuniu and Zhang, Yingya and Shen, Yujun and Zhao, Deli and Zhou, Jingren},
  journal={Advances in Neural Information Processing Systems},
  volume={36},
  pages={7594--7611},
  year={2023}
}

@article{li2025magicmotion,
  title={Magicmotion: Controllable video generation with dense-to-sparse trajectory guidance},
  author={Li, Quanhao and Xing, Zhen and Wang, Rui and Zhang, Hui and Dai, Qi and Wu, Zuxuan},
  journal={arXiv preprint arXiv:2503.16421},
  year={2025}
}

@article{geyer2023tokenflow,
  title={Tokenflow: Consistent diffusion features for consistent video editing},
  author={Geyer, Michal and Bar-Tal, Omer and Bagon, Shai and Dekel, Tali},
  journal={arXiv preprint arXiv:2307.10373},
  year={2023}
}

@inproceedings{wang2025cinemaster,
  title={Cinemaster: A 3d-aware and controllable framework for cinematic text-to-video generation},
  author={Wang, Qinghe and Luo, Yawen and Shi, Xiaoyu and Jia, Xu and Lu, Huchuan and Xue, Tianfan and Wang, Xintao and Wan, Pengfei and Zhang, Di and Gai, Kun},
  booktitle={Proceedings of the Special Interest Group on Computer Graphics and Interactive Techniques Conference Conference Papers},
  pages={1--10},
  year={2025}
}

@inproceedings{liu2022bevfusion,
  title={BEVFusion: Multi-Task Multi-Sensor Fusion with Unified Bird's-Eye View Representation},
  author={Liu, Zhijian and Tang, Haotian and Amini, Alexander and Yang, Xingyu and Mao, Huizi and Rus, Daniela and Han, Song},
  booktitle={IEEE International Conference on Robotics and Automation (ICRA)},
  year={2023}
}

@article{LCTGen,
  title={Language conditioned traffic generation},
  author={Tan, Shuhan and Ivanovic, Boris and Weng, Xinshuo and Pavone, Marco and Kraehenbuehl, Philipp},
  journal={arXiv preprint arXiv:2307.07947},
  year={2023}
}

@article{jiang2024dive,
  title={Dive: Dit-based video generation with enhanced control},
  author={Jiang, Junpeng and Hong, Gangyi and Zhou, Lijun and Ma, Enhui and Hu, Hengtong and Zhou, Xia and Xiang, Jie and Liu, Fan and Yu, Kaicheng and Sun, Haiyang and others},
  journal={arXiv preprint arXiv:2409.01595},
  year={2024}
}

@inproceedings{kim2021drivegan,
  title={DriveGAN: Towards a Controllable High-Quality Neural Simulation},
  author={Kim, Seung Wook and Philion, Jonah and Torralba, Antonio and Fidler, Sanja},
  booktitle={Proceedings of the IEEE/CVF Conference on Computer Vision and Pattern Recognition},
  pages={5820--5829},
  year={2021}
}

@article{wu2024drivescape,
  title={Drivescape: Towards high-resolution controllable multi-view driving video generation},
  author={Wu, Wei and Guo, Xi and Tang, Weixuan and Huang, Tingxuan and Wang, Chiyu and Chen, Dongyue and Ding, Chenjing},
  journal={arXiv preprint arXiv:2409.05463},
  year={2024}
}

@inproceedings{umgen,
  title={Generating Multimodal Driving Scenes via Next-Scene Prediction},
  author={Wu, Yanhao and Zhang, Haoyang and Lin, Tianwei and Huang, Lichao and Luo, Shujie and Wu, Rui and Qiu, Congpei and Ke, Wei and Zhang, Tong},
  booktitle={Proceedings of the Computer Vision and Pattern Recognition Conference},
  pages={6844--6853},
  year={2025}
}

@article{li2203bevformer,
  title={Bevformer: learning bird's-eye-view representation from lidar-camera via spatiotemporal transformers},
  author={Li, Zhiqi and Wang, Wenhai and Li, Hongyang and Xie, Enze and Sima, Chonghao and Lu, Tong and Yu, Qiao and Dai, Jifeng},
  journal={IEEE Transactions on Pattern Analysis and Machine Intelligence},
  year={2024},
  publisher={IEEE}
}
\end{document}